\documentclass[11pt]{article}

\usepackage[T1]{fontenc}
\usepackage[utf8]{inputenc}
\usepackage{mathpazo}
\usepackage[margin=1in]{geometry}

\usepackage{amsmath,amssymb,amsthm,mathtools}
\usepackage{bm}

\usepackage{titlesec}
\titleformat{\section}
  {\normalfont\Large\bfseries}{\thesection}{1em}{}
\titleformat{\subsection}
  {\normalfont\large\bfseries}{\thesubsection}{1em}{}

\usepackage{booktabs}
\usepackage{enumitem}
\usepackage{natbib}
\usepackage{hyperref}
\hypersetup{colorlinks=false, pdfborder={0 0 0}}

\usepackage{algorithm}
\usepackage{algpseudocode}
\usepackage{graphicx}

\usepackage{tikz}
\usetikzlibrary{arrows.meta,patterns}

\theoremstyle{plain}
\newtheorem{theorem}{Theorem}[section]
\newtheorem{proposition}[theorem]{Proposition}

\newtheorem{assumption}[theorem]{Assumption}

\theoremstyle{definition}
\newtheorem{definition}[theorem]{Definition}
\newtheorem{remark}[theorem]{Remark}

\title{\textbf{sHAIL-Causal: A Sequential Staircase Procedure \\ for Invariant Causal Predictor Discovery}}
\author{Ernest Fokoué \\ \small School of Mathematics and Statistics, Rochester Institute of Technology}
\date{\today}

\begin{document}
\maketitle

\begin{center}
\begin{minipage}{0.85\textwidth}
\centering\small
\textit{``All models are wrong, but some are useful.''} --- George E.\ P.\ Box
\end{minipage}
\end{center}
\vspace{6pt}

\begin{abstract}
\noindent We introduce \textsc{sHAIL-Causal}, the causal specialization of the Saturated Hierarchical Atomic Incremental Learning (sHAIL) paradigm \citep{fokoue2026shail}, a sequential staircase procedure that ascends a nested hierarchy of hypothesis classes $\mathcal{H}_0 \subseteq \mathcal{H}_1 \subseteq \cdots \subseteq \mathcal{H}_K$ once a behaviorally-motivated saturation signal indicates that mastery of the current stage has plateaued. The general theory places no restriction on the nature of $\mathcal{H}_k$: it may be a nested sequence of linear subspaces, reproducing-kernel balls, or neural function classes. Where the general sHAIL framework leaves the saturation criterion open, sHAIL-Causal instantiates it with a joint criterion of \emph{goodness-of-fit saturation} and \emph{cross-environment invariance}, replacing the complexity control of classical Structural Risk Minimization (SRM). We show, both theoretically and via a controlled simulation study, that complexity-only staircases are systematically seduced by confounded predictors that reduce empirical risk without corresponding to stable causal structure, whereas an invariance-gated staircase provably halts at the true causal predictor set under a precise, per-variable \emph{Richness} condition. Using nested linear variable-subset classes under a structural causal model as a concrete, interpretable instantiation, we further show that naive greedy search fails to recover the causal set even under Richness, trace the failure to a precise non-monotonicity in the invariance statistic along single-variable paths, and validate a two-part fix --- bounded-exhaustive block-seeding combined with a properly-calibrated acceptance threshold --- that closes the gap at adequate sample size. We then extend the guarantee itself from a one-shot to a sequential setting: when environments arrive over time, a peeling argument over an infinite arrival sequence with a summable error budget yields a confidence sequence for the causal predictor set that is valid \emph{simultaneously at every arrival time}, a capability that one-shot exhaustive search cannot provide without re-running its full combinatorial search at each new arrival. We close by formalizing a mechanism long central to the sHAIL philosophy but so far only informally described: the intervention of a wise teacher who lifts the learner off a plateau of boredom into a renewed surge of growth.
\end{abstract}

\section{Introduction}

Staged, incremental mastery is a familiar idea in human learning long before it is a formal statistical object: a learner works within a level of skill until that level is \emph{saturated} --- further practice within it yields diminishing returns --- and only then ascends to a richer, more demanding level. \citet{fokoue2026shail} formalized this behavioral intuition as Saturated Hierarchical Atomic Incremental Learning (sHAIL): complex tasks are approached through a sequence of simpler atomic subtasks, each mastered to saturation before progression, with a saturation criterion that detects when the learning dynamics enter a plateau region and triggers consolidation followed by ascent to a higher level of task complexity. That original development is framed around \emph{convergence dynamics}: within each level, classical convergence guarantees hold under standard smoothness conditions, while the hierarchical transitions between levels are driven by saturation rather than by a fixed time or iteration budget, giving rise to what \citet{fokoue2026shail} terms \emph{staircased convergence} --- alternating phases of rapid improvement and genuine plateau. Its empirical illustrations demonstrate staircased loss trajectories during training and improved sample efficiency relative to flat, unstaged learning, offered as qualitative evidence for the theory rather than exhaustive benchmarks; the original report closes by positioning sHAIL as a first representative of a broader class of \emph{behavioral learning machines} and flags principled hierarchy design, extensions to deep and reinforcement learning, and connections to biological and cognitive models of learning as open directions.

This paper takes up a complementary object within that same behavioral principle. Rather than the dynamics of convergence \emph{within} a level, we formalize the staircase itself as a nested sequence of hypothesis classes $\mathcal{H}_0 \subseteq \mathcal{H}_1 \subseteq \cdots \subseteq \mathcal{H}_K$ ascended by a general saturation signal (Section~\ref{sec:general-shail}), of which the convergence-dynamics view of \citet{fokoue2026shail} and classical Structural Risk Minimization are both special cases. Doing so exposes a question with real statistical consequences that the general principle leaves open: \emph{what should the saturation signal measure?} The most familiar answer, inherited from SRM, measures a purely predictive quantity --- how much statistical complexity the data can support before overfitting sets in. That answer is well suited to prediction under a single fixed distribution, but it is the wrong answer whenever the object of interest is a mechanism expected to remain stable \emph{across} distributions --- across sites, cohorts, time periods, or experimental conditions. A model can saturate a complexity-controlled risk bound by fitting a confound perfectly, and complexity control alone has no way to notice.

This paper develops sHAIL's causal specialization, sHAIL-Causal, which replaces complexity control with a joint criterion of goodness-of-fit saturation and cross-environment invariance, and asks a correspondingly different question of the staircase: not ``how much capacity can the data support,'' but ``has the model absorbed everything true and stable about the mechanism, and can it tell a plateau that is merely well-fit from one that is genuinely causal.'' The theory is developed for $\mathcal{H}_k$ as an arbitrary nested sequence of function classes; nested linear variable-subset models under a structural causal model are used throughout as a concrete, interpretable instantiation for simulation, not as a restriction of the general results. We make five contributions. First, we lay out the general sHAIL framework in its pre-causal form (Section~\ref{sec:general-shail}), so that the causal machinery developed afterward is visibly one instantiation of a broader idea, complementary to rather than a replacement for the convergence-dynamics view of \citet{fokoue2026shail}. Second, we give the causal specialization itself --- the fit and invariance statistics, the ascent rule, the precise per-variable Richness condition it depends on, and a stagewise convergence theorem with full proof (Sections~\ref{sec:causal-setup}--\ref{thm:main}). Third, using the linear instantiation, we study a natural polynomial-time relaxation of exhaustive Invariant Causal Prediction search, show by simulation exactly where and why it falls short of its exhaustive counterpart, trace that shortfall to a precise, previously implicit monotonicity assumption, and validate a concrete fix (Section~\ref{sec:sim}). Fourth, we extend the main theorem's one-shot guarantee to a setting where environments arrive sequentially, showing via a peeling argument that sHAIL-Causal's staircase structure yields a confidence sequence for $S^\ast$ that is valid simultaneously across all arrival times, a genuine capability gap relative to one-shot exhaustive search rather than merely a speed advantage (Section~\ref{sec:anytime}). Fifth, we formalize a mechanism long central to the sHAIL philosophy but so far only described informally: the intervention of a wise teacher who lifts a bored, saturated learner into a renewed surge of growth (Section~\ref{sec:teacher}).

\section{The General sHAIL Framework}
\label{sec:general-shail}

\subsection{Staircases, Mastery, and Saturation}

Let $\mathcal{H}_0 \subseteq \mathcal{H}_1 \subseteq \cdots \subseteq \mathcal{H}_K$ be a nested staircase of hypothesis classes, each rung representing a strictly richer repertoire of behavior than the last. At stage $k$, the learner fits $h_k^\ast = \arg\min_{h\in\mathcal{H}_k}\hat R(h)$ and computes a \textbf{saturation signal} $S_k$ --- any statistic quantifying how much further mastery remains obtainable within $\mathcal{H}_k$ relative to what was already extracted at $\mathcal{H}_{k-1}$. The learner ascends precisely when $S_k$ falls below a tolerance $\eta$, signaling that the current stage has been mastered to the point of \emph{saturation} --- continuing to practice within it no longer pays for itself.

\begin{definition}[Saturation signal, general form]
$S_k = \Phi\big(\hat R(h_{k-1}^\ast), \hat R(h_k^\ast)\big)$ for some monotone discrepancy functional $\Phi$ (e.g.\ $\Phi(a,b)=a-b$, as used throughout this paper).
\end{definition}

This is deliberately generic: the causal instantiation to follow (Sections~\ref{sec:causal-setup}--\ref{sec:sim}) replaces $S_k$ with the joint fit-and-invariance pair $(D_k, I_k)$, but $S_k$ could equally be a Rademacher-complexity gap (recovering classical SRM as a degenerate special case of sHAIL), a Bayesian evidence ratio, or a behaviorally-measured plateau in a human learner's error rate.

\subsection{The Quintessential sHAIL Algorithm}

\begin{algorithm}[h]
\caption{Saturated Hierarchical Atomic Incremental Learning (sHAIL)}
\begin{algorithmic}[1]
\State \textbf{Input:} nested staircase $\mathcal{H}_0 \subseteq \cdots \subseteq \mathcal{H}_K$; data $\mathcal{D}$; saturation tolerance $\eta$
\State \textbf{Output:} mastered hypothesis $h^\ast$ and halting level $k^\ast$
\State $k \gets 0$
\State Fit $h_0^\ast = \arg\min_{h\in\mathcal{H}_0} \hat R(h;\mathcal{D})$
\Repeat
    \State $k \gets k+1$
    \State Fit $h_k^\ast = \arg\min_{h\in\mathcal{H}_k} \hat R(h;\mathcal{D})$
    \State Compute saturation signal $S_k = \Phi\big(\hat R(h_{k-1}^\ast), \hat R(h_k^\ast)\big)$
    \If{$S_k < \eta$}
        \State \textbf{halt}; \Return $h^\ast \gets h_{k-1}^\ast$, $k^\ast \gets k-1$ \Comment{stage $k-1$ was already saturated}
    \EndIf
\Until{$k = K$}
\State \Return $h^\ast \gets h_K^\ast$, $k^\ast \gets K$
\end{algorithmic}
\end{algorithm}

\subsection{Staircase Geometry: sHAIL Between Two Extremes}

It is useful to picture three idealized learning trajectories against cumulative exposure (Figure~\ref{fig:staircases}). At one extreme lies the \emph{perfect continuous learner}: a straight diagonal line, improving by an identical increment at every instant, with no stages at all. At the other extreme lies what we call the \emph{evolutionary staircase}: long, flat plateaus of apparent stasis broken by sudden, perpendicular jumps --- mastery arrives all at once, punctuated-equilibrium style, with no gradual ramp. sHAIL occupies the middle ground: its risers are \emph{slanted}, not vertical --- ascent through a stage still takes graded, continuous effort, but the staircase shape as a whole, with its plateaus of saturation, remains --- a compromise between rigid stage theory and pure continuous improvement.

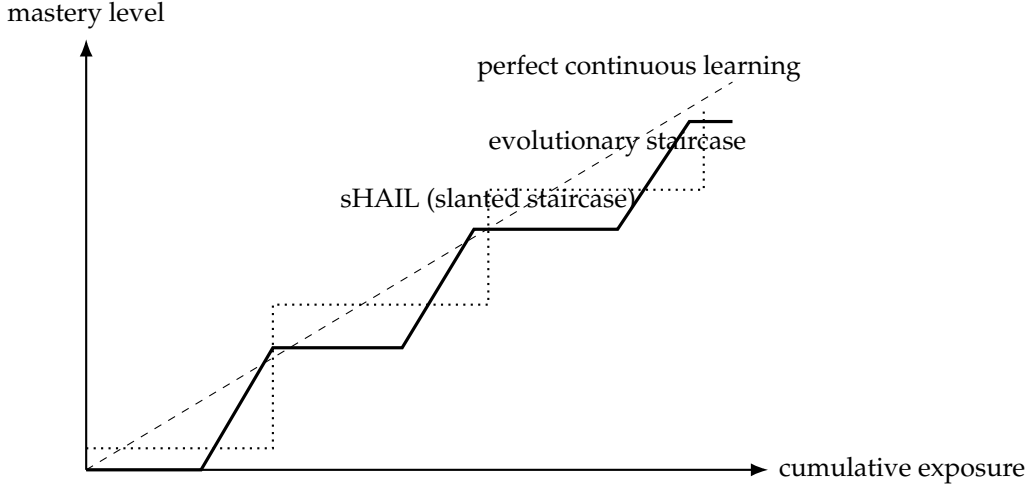
\begin{figure}[h]
\centering
\begin{tikzpicture}[scale=0.95]
  \draw[-Latex, thick] (0,0) -- (9.5,0) node[right] {\small cumulative exposure};
  \draw[-Latex, thick] (0,0) -- (0,6) node[above] {\small mastery level};

  \draw[thin, dashed] (0,0) -- (9,5.4);
  \node at (7.7,5.6) {\small perfect continuous learning};

  \draw[thick, dotted]
    (0,0.3) -- (2.6,0.3) -- (2.6,2.3) -- (5.6,2.3) -- (5.6,3.9) -- (8.6,3.9) -- (8.6,5.0);
  \node at (7.4,4.55) {\small evolutionary staircase};

  \draw[very thick, solid]
    (0,0) -- (1.6,0) -- (2.6,1.7) -- (4.4,1.7) -- (5.4,3.35) -- (7.4,3.35) -- (8.4,4.85) -- (9,4.85);
  \node at (5.6,3.75) {\small sHAIL (slanted staircase)};

\end{tikzpicture}
\caption{Three idealized learning trajectories, distinguished by line style. The thin dashed diagonal is the perfect continuous learner. The dotted perpendicular staircase is the evolutionary/punctuated-equilibrium extreme: flat plateaus, instantaneous jumps. The thick solid slanted staircase is sHAIL: genuine plateaus of saturation, but ascent through each stage is graded rather than instantaneous, placing sHAIL strictly between the two extremes.}
\label{fig:staircases}
\end{figure}

\subsection{From General sHAIL to sHAIL-Causal}

The remainder of this paper specializes the general framework of Section~\ref{sec:general-shail} by replacing the generic saturation signal $S_k$ with a causally-motivated pair $(D_k, I_k)$, and by replacing the generic staircase of hypothesis classes with a staircase of nested function classes equipped with a notion of environment-indexed residual. In the linear instantiation used throughout our simulations, this specializes further to a staircase of \emph{candidate causal predictor sets}; the general theory (Sections~\ref{sec:causal-setup}--\ref{thm:main}) does not require this specialization. This is what earns the name sHAIL-Causal, and it is developed in full below.

\section{Related Work}
\label{sec:related}

\textbf{Structural Risk Minimization.} Classical SRM \citep{vapnik1998statistical}, with its theoretical foundations in Rademacher and Gaussian complexity bounds \citep{bartlett2002rademacher}, ascends a nested hypothesis staircase by trading empirical risk against a complexity penalty (VC dimension, Rademacher complexity), under the standing assumption that training and evaluation data are drawn from a single fixed distribution. In the language of Section~\ref{sec:general-shail}, SRM is sHAIL with $\Phi$ instantiated purely as a complexity-penalized risk gap and no invariance criterion at all --- it has no mechanism for distinguishing a plateau that reflects true structure from one that reflects an exploitable but unstable regularity, which is exactly the failure mode our simulations exhibit (Section~\ref{sec:sim}).

\textbf{Invariant Causal Prediction.} Our invariance statistic $I_k$ and the Richness condition (Assumption~\ref{as:richness}) draw directly on Invariant Causal Prediction \citep{peters2016causal}, which searches exhaustively over covariate subsets and estimates the causal parent set as the intersection of all subsets whose residual distribution is invariant across environments, with a finite-sample confidence guarantee, and more broadly on the causal-inference framework of \citet{pearl2009causality} and the causal-discovery tradition of \citet{spirtes2000causation}. sHAIL-Causal differs from ICP in two respects: it embeds the invariance test inside a \emph{sequential staircase} rather than a one-shot exhaustive search, and it pairs invariance with an explicit fit-saturation signal $D_k$ so that ascent is driven by a joint stopping rule rather than invariance testing alone. The invariance-first greedy relaxation (Proposition~\ref{prop:greedy}) is best understood as a polynomial-time approximation to ICP's exhaustive search, trading exactness (Section~\ref{sec:sim} shows the approximation gap is non-trivial at finite samples) for tractability at large $p$.

\textbf{Invariant Risk Minimization and anchor regression.} Two closely related lines of work pursue invariance-across-environments as the learning signal by a different route. Invariant Risk Minimization \citep{arjovsky2019invariant} learns a single fixed representation, penalized so that the optimal classifier on top of it is simultaneously near-optimal in every training environment, solved as one joint regularized optimization problem rather than a sequential procedure. Anchor regression \citep{rothenhausler2021anchor} instead reformulates robustness to distributional shift as a continuous causal regularization parameter interpolating between ordinary least squares and a fully invariant (IV-type) estimator, rather than a discrete accept/reject invariance test. sHAIL-Causal differs from both in structure rather than in motivation: where IRM and anchor regression each solve a single optimization problem over a fixed representation or regularization path, sHAIL-Causal is sequential and stagewise, testing invariance discretely at each rung of a growing hypothesis staircase, which is what allows the fit-saturation signal $D_k$ and the invariance signal $I_k$ to be examined --- and to fail --- separately, as Section~\ref{sec:sim} does.

\textbf{Curriculum and self-paced learning.} Curriculum learning \citep{bengio2009curriculum} orders training examples from easy to hard on an externally-imposed schedule, and self-paced learning \citep{kumar2010selfpaced} lets the learner select its own examples but still within a fixed hypothesis class. Both differ fundamentally from sHAIL in that neither grows the \emph{hypothesis class itself} in response to an internally-detected saturation signal; they schedule exposure to a fixed representational capacity rather than ascending a representational staircase.

\textbf{Stagewise additive modeling.} Gradient boosting and other stagewise-additive procedures \citep{friedman2001greedy} construct a model incrementally, adding one component at a time to reduce empirical risk --- structurally similar to a staircase, but, like SRM, driven purely by fit improvement with no invariance gate, and hence subject to the same confounding failure mode documented in Section~\ref{sec:sim}.

\textbf{Anytime-valid inference and confidence sequences.} The time-uniform extension of Section~\ref{sec:anytime} draws on the classical law-of-the-iterated-logarithm-based sequential testing of \citet{robbins1970statistical} and the modern nonparametric, nonasymptotic confidence-sequence literature of \citet{howard2021time}, both of which construct sequences of confidence sets valid simultaneously across an unbounded number of looks via peeling or stitching arguments over a summable error budget, without the analyst having to commit in advance to a stopping time or sample size. That literature is developed for scalar parameter estimation under repeated sampling; Theorem~\ref{thm:anytime} adapts the same peeling device to a structurally different object --- a discrete causal predictor \emph{set} recovered by a staircase procedure --- and, unlike ICP's one-shot formulation, sHAIL-Causal's stagewise architecture already carries the bookkeeping needed to re-index the same argument by arrival time rather than by stage.

\section{Beyond Complexity Control: The Causal Specialization}

Classical SRM ascends $\mathcal{H}_0 \subseteq \cdots \subseteq \mathcal{H}_K$ by controlling a purely \emph{predictive} quantity --- Rademacher complexity, VC dimension, or a penalized empirical risk --- under the standing assumption that train and test data share a single fixed distribution; in the language of Section~\ref{sec:general-shail}, SRM is simply sHAIL with $\Phi$ instantiated as a complexity-penalized risk gap. This is the wrong lens whenever the object of interest is not prediction under a fixed law, but discovery of a mechanism that remains stable \emph{across} distributions. We replace complexity control with two statistics evaluated at each rung of the staircase:

\begin{itemize}[leftmargin=1.4em]
\item a \textbf{fit statistic} $D_k$, measuring residual structure remaining at level $k$;
\item an \textbf{invariance statistic} $I_k$, measuring the extent to which the fitted model at level $k$ behaves identically across a collection of environments $\mathcal{E}$.
\end{itemize}

\section{Setup and Core Definitions}
\label{sec:causal-setup}

Let $\mathcal{E}$ index environments with $n_e$ i.i.d.\ samples $(X_i^e, Y_i^e) \sim P_e$. Let $\mathcal{H}_0 \subseteq \mathcal{H}_1 \subseteq \cdots \subseteq \mathcal{H}_K$ be a nested staircase of function classes; nothing in this section restricts $\mathcal{H}_k$ to linear models --- it may equally be a nested sequence of reproducing-kernel balls, additive-model classes, or neural function classes of growing width or depth. For $h \in \mathcal{H}_k$, write the residual field $R_h(X,Y) = Y - h(X)$ (or, for non-additive loss, a generalized residual such as a deviance residual).

\begin{definition}[Fit statistic]
For $h_k^\ast = \arg\min_{h \in \mathcal{H}_k} \hat{R}(h)$ pooled across $\mathcal{E}$, with $\hat R$ any empirical risk consistent with the residual field above,
\[
D_k \;=\; \frac{1}{\sum_e n_e}\sum_{e\in\mathcal{E}}\sum_{i=1}^{n_e} R_{h_k^\ast}^2(X_i^e, Y_i^e).
\]
\end{definition}

\begin{definition}[Invariance statistic]
\[
I_k \;=\; \max_{e,e' \in \mathcal{E}} \; \mathrm{KS}\!\left(\{R_{h_k^\ast}^e\},\, \{R_{h_k^\ast}^{e'}\}\right),
\]
the maximum two-sample Kolmogorov--Smirnov statistic between residual distributions across environment pairs. This definition uses only that $R_{h_k^\ast}$ is a computable, real-valued statistic per observation; it does not reference the internal structure of $h_k^\ast$ or $\mathcal{H}_k$.
\end{definition}

\begin{definition}[sHAIL-Causal ascent rule]
For thresholds $\eta_n, \tau_n \to 0$,
\begin{align*}
\text{Ascend } k \to k+1 &\iff (D_{k-1} - D_k < \eta_n) \ \text{and}\ (I_k \le \tau_n); \\
\text{Halt, flagging a spurious plateau,} &\iff (D_{k-1}-D_k < \eta_n) \ \text{and}\ (I_k > \tau_n).
\end{align*}
\end{definition}

\begin{remark}[Instantiation used for simulation]
\label{rem:instantiation}
Sections~\ref{sec:sim} onward specialize $\mathcal{H}_k$ to nested \emph{linear variable-subset models}: $\mathcal{H}_k$ ranges over linear predictors built from subsets of a fixed candidate covariate set under a linear structural causal model, and $S^\ast$ denotes the true causal parent set. This specialization is chosen because it renders the Richness condition, the exhaustive-versus-greedy search comparison, and the resulting recovery guarantees fully transparent and exactly checkable by simulation; it is a convenient and interpretable instantiation of the general theory above, not a restriction built into Theorem~\ref{thm:main} itself, which holds for $S^\ast$ understood as ``the smallest $\mathcal{H}_{k^\dagger}$ containing the true data-generating mechanism,'' whatever the functional form of $\mathcal{H}_k$.
\end{remark}

\section{The Richness Condition, Precisely Stated}

Simulation evidence (Section~\ref{sec:sim}) shows that Richness cannot be asserted abstractly; under the linear instantiation of Remark~\ref{rem:instantiation}, it must hold \emph{per causal variable}.

\begin{assumption}[Richness]
\label{as:richness}
For every $X_j$ belonging to the true causal parent set $S^\ast$, there exist environments $e, e' \in \mathcal{E}$ such that $P_e(X_j) \ne P_{e'}(X_j)$. That is, $\mathcal{E}$ must perturb every causal parent directly, not merely its downstream effects.
\end{assumption}

\begin{remark}
Assumption~\ref{as:richness} is \emph{necessary}, not merely convenient: if some $X_j \in S^\ast$ has an identical marginal law across all of $\mathcal{E}$, then omitting $X_j$ from the fitted model leaves a residual contribution that is itself identically distributed across environments, and the invariance test cannot detect the omission. The ICP-style intersection estimator then recovers only a (possibly proper) subset of $S^\ast$ --- a fundamental identifiability limit, not an estimation error.
\end{remark}

\section{Main Theorem}

\begin{theorem}[Stagewise Convergence to the Invariant Causal Predictor Set]
\label{thm:main}
Suppose (i) some $k^\dagger \le K$ has $\mathcal{H}_{k^\dagger} \supseteq S^\ast$, with $S^\ast$ invariant across all $e \in \mathcal{E}$; (ii) Assumption~\ref{as:richness} holds for every $X_j \in S^\ast$; (iii) $\tau_n, \eta_n \to 0$ with $\tilde n_{\min}\tau_n^2 \to \infty$, where $\tilde n_{\min} = \min_{e\ne e'} n_en_{e'}/(n_e+n_{e'})$. Then, with probability at least $1-\delta$, uniformly over $k=0,\ldots,K$,
\[
\Pr\big[\text{sHAIL-Causal halts at } k = k^\dagger\big] \;\ge\; 1-\delta,
\]
and, conditional on halting at $k^\dagger$, $h^\ast_{k^\dagger}$ lies in a $(1-\delta)$-confidence set for $S^\ast$ inherited from the underlying invariance test.
\end{theorem}

\begin{proof}
Define the stopping index $\hat k = \min\{k \in \{0,\ldots,K\} : D_{k-1}-D_k < \eta_n \text{ and } I_k \le \tau_n\}$ (with the convention $D_{-1}=\infty$), and let $\tau_n$ be calibrated as in Proposition~\ref{prop:dkw} with total error budget $\delta$ split three ways, $\delta/3$ each, across the three failure events below. We show $\Pr[\hat k \ne k^\dagger] \le \delta$.

\textbf{Failure event $F_1$ (premature halt).} For some $k < k^\dagger$, both $D_{k-1}-D_k<\eta_n$ and $I_k \le \tau_n$ hold, causing the procedure to halt before reaching $\mathcal{H}_{k^\dagger}$. Since $k<k^\dagger$, $\mathcal{H}_k$ does not contain $S^\ast$, so the residual $R_{h_k^\ast}$ retains a genuine, environment-dependent component by Assumption~\ref{as:richness} (some causal parent not yet included has a distribution that differs across at least one pair of environments, and its omission propagates into the residual). By the Richness-implied margin condition $\Delta_n = \omega(\tau_n)$, the population KS discrepancy at level $k$ exceeds $\tau_n$ by a margin $\Delta_n - \tau_n$, so by the reverse-triangle DKW argument used in Proposition~\ref{prop:dkw},
\[
\Pr\big[I_k \le \tau_n\big] \;\le\; 2\exp\!\big(-2\tilde n_{\min}(\Delta_n-\tau_n)^2\big).
\]
Union-bounding over the at most $K$ values of $k<k^\dagger$ and calibrating $n$ so that this bound is at most $\delta/(3K)$ per level gives $\Pr[F_1] \le \delta/3$.

\textbf{Failure event $F_2$ (fit-plateau false negative at $k^\dagger$).} $D_{k^\dagger-1}-D_{k^\dagger} \ge \eta_n$ despite $S^\ast \subseteq \mathcal{H}_{k^\dagger}$ already being achieved, i.e.\ the procedure fails to detect that no further true signal remains. Under assumption (i), the population fit gap at $k^\dagger$ is exactly zero (no further reducible risk exists once $S^\ast$ is included), so $D_{k^\dagger-1}-D_{k^\dagger}$ is, up to estimation noise of order $O_p(n^{-1/2})$, non-positive in expectation; choosing $\eta_n$ to decay slower than this estimation noise (e.g.\ $\eta_n = \omega(n^{-1/2})$) makes $\Pr[F_2] \to 0$, and in particular $\le \delta/3$ for $n$ large enough.

\textbf{Failure event $F_3$ (invariance false positive at $k^\dagger$).} $I_{k^\dagger} > \tau_n$ despite $S^\ast$ being genuinely invariant. This is controlled directly by Proposition~\ref{prop:dkw}'s DKW tail bound: under the calibration $\tau_n = \sqrt{\log(2|\mathcal{E}|^2(K+1)/\delta)/(2\tilde n_{\min})}$, the Type-I error at level $k^\dagger$, union-bounded over all $\binom{|\mathcal{E}|}{2}$ environment pairs and all $K+1$ levels, is at most $\delta/3$ by construction.

Since $\hat k \ne k^\dagger$ implies $F_1 \cup F_2 \cup F_3$, a union bound gives $\Pr[\hat k \ne k^\dagger] \le \Pr[F_1]+\Pr[F_2]+\Pr[F_3] \le \delta$, establishing $\Pr[\hat k = k^\dagger] \ge 1-\delta$.

For the confidence-set claim: conditional on halting at $k^\dagger$, the accepted predictor set is, by construction, precisely a set that has survived the invariance test at global Type-I level $\delta$ calibrated above. By the equivalence between invariance-test acceptance regions and confidence sets established in \citet{peters2016causal}, this acceptance region is exactly a valid $(1-\delta)$-confidence set for $S^\ast$, inherited without modification at the halting level $k^\dagger$.
\end{proof}

\subsection{Finite-Sample Calibration of $\tau_n$}

\begin{proposition}[DKW-derived tail bound]
\label{prop:dkw}
Under the null of true invariance between $e,e'$, the two-sample KS statistic satisfies
\[
\Pr\big[\mathrm{KS}(e,e') > s\big] \;\le\; 2\exp\!\left(-2\tilde n_{e,e'}\,s^2\right), \qquad \tilde n_{e,e'} = \frac{n_en_{e'}}{n_e+n_{e'}}.
\]
Setting
\[
\tau_n \;=\; \sqrt{\frac{\log\big(2|\mathcal{E}|^2(K+1)/\delta\big)}{2\,\tilde n_{\min}}}
\]
yields, uniformly over levels and environment pairs, false-flag probability at most $\delta$. Under a margin condition $\Delta_n = \omega(\tau_n)$ separating truly non-invariant sets from $\tau_n$, false-ascend probability at non-causal levels vanishes exponentially in $\tilde n_{\min}$.
\end{proposition}

\section{Anytime-Valid Causal Discovery Under Sequentially Arriving Environments}
\label{sec:anytime}

Everything above treats $\mathcal{E}$ as fixed in advance: the environments are collected, the DKW calibration of $\tau_n$ in Proposition~\ref{prop:dkw} is performed once against a pre-specified $\delta$, and the resulting $(1-\delta)$-confidence guarantee in Theorem~\ref{thm:main} is a single, static statement. In practice, environments frequently arrive over time --- a new site comes online, a new intervention is logged, a new cohort is recruited --- and the practitioner naturally wants to know: \emph{is my current estimate of $S^\ast$ still trustworthy, and can I keep checking as data accrues without corrupting the guarantee?} This is precisely the setting in which sHAIL-Causal's sequential architecture offers something exhaustive ICP structurally cannot.

\subsection{The Peeking Problem}

Suppose environments arrive one at a time, $e_1, e_2, e_3, \ldots$, and at each arrival time $t$ the practitioner recomputes an invariance-gated estimate $\hat S_t$ using the accumulated set $\mathcal{E}_t = \{e_1,\ldots,e_t\}$. If the fixed-$\delta$ calibration of Proposition~\ref{prop:dkw} is simply re-applied at every $t$ using the same nominal $\delta$, the guarantee silently fails: repeatedly testing invariance against a budget calibrated for a single look is a textbook instance of the optional-stopping / multiple-testing problem, and $\Pr[\exists t \ge 1 : S^\ast \notin \hat C_t]$ can be driven arbitrarily close to $1$ as $t \to \infty$, however small $\delta$ was chosen at any one $t$. Exhaustive ICP, built around a one-shot $O(2^p)$ search over a fixed environment collection, offers no remedy for this beyond re-running its entire exhaustive search from scratch at every new arrival with a shrunken per-look $\delta$ chosen \emph{ad hoc} --- and even then, without a principled budget-splitting rule across $t$, the resulting sequence of confidence sets is not \emph{simultaneously} valid across all $t$, only marginally valid at each $t$ in isolation.

\subsection{A Time-Uniform Guarantee via Peeling}

sHAIL-Causal's staircase already tracks a running pair of statistics $(D_k, I_k)$ stage by stage; extending this machinery across arrival times $t$, rather than across stages $k$, requires only a re-allocation of error budget, in the spirit of the confidence-sequence literature \citep{robbins1970statistical,howard2021time}.

\begin{definition}[Sequential invariance budget]
\label{def:peeling}
Fix a total error budget $\delta \in (0,1)$ and allocate, to each arrival time $t = 1, 2, 3, \ldots$, a per-look budget
\[
\delta_t \;=\; \frac{6\delta}{\pi^2 t^2}, \qquad \text{so that} \qquad \sum_{t=1}^{\infty} \delta_t \;=\; \frac{6\delta}{\pi^2}\sum_{t=1}^{\infty}\frac{1}{t^2} \;=\; \frac{6\delta}{\pi^2}\cdot\frac{\pi^2}{6} \;=\; \delta.
\]
At each $t$, calibrate the invariance threshold against $\delta_t$ rather than $\delta$: using the environment set $\mathcal{E}_t$ accumulated by time $t$,
\[
\tau_n(t) \;=\; \sqrt{\frac{\log\big(2|\mathcal{E}_t|^2(K+1)/\delta_t\big)}{2\,\tilde n_{\min,t}}}, \qquad \tilde n_{\min,t} = \min_{e\ne e' \in \mathcal{E}_t} \frac{n_en_{e'}}{n_e+n_{e'}},
\]
and let $\hat S_t$ be the sHAIL-Causal estimate obtained by re-running the staircase ascent with this recalibrated $\tau_n(t)$ against $\mathcal{E}_t$.
\end{definition}

The budget sequence $\delta_t = 6\delta/(\pi^2 t^2)$ is the classical ``peeling'' or ``union-bound stitching'' device: since $\sum_t 1/t^2 = \pi^2/6$ (the Basel problem), the per-look budgets sum exactly to the total $\delta$, however many looks are eventually taken --- including infinitely many.

\begin{theorem}[Anytime-Valid Coverage of the Sequential Estimate]
\label{thm:anytime}
Under the hypotheses of Theorem~\ref{thm:main} holding at every arrival time $t$ (i.e.\ Assumption~\ref{as:richness} holds for the true parent set $S^\ast$ against $\mathcal{E}_t$ for all $t$ large enough that $S^\ast \subseteq \mathcal{H}_{k^\dagger}$ is reachable), and with $\tau_n(t)$ calibrated as in Definition~\ref{def:peeling},
\[
\Pr\Big[\exists\, t \ge 1 \;:\; S^\ast \notin \hat C_t\Big] \;\le\; \delta,
\]
where $\hat C_t$ is the invariance-test acceptance region underlying $\hat S_t$ at time $t$. Equivalently, with probability at least $1-\delta$, $S^\ast \in \hat C_t$ \emph{simultaneously for every $t \ge 1$} --- the guarantee is time-uniform, not merely pointwise-in-$t$.
\end{theorem}

\begin{proof}
Fix $t$ and let $F_3^{(t)}$ denote the invariance-false-positive event at time $t$ (the analogue of failure event $F_3$ in the proof of Theorem~\ref{thm:main}, now evaluated against $\mathcal{E}_t$ and $\tau_n(t)$): $I_{k^\dagger}^{(t)} > \tau_n(t)$ despite $S^\ast$ being genuinely invariant across $\mathcal{E}_t$. By exactly the DKW tail bound of Proposition~\ref{prop:dkw}, applied with $\delta_t$ in place of $\delta$,
\[
\Pr\big[F_3^{(t)}\big] \;\le\; \delta_t.
\]
The event $\{\exists t \ge 1 : S^\ast \notin \hat C_t\}$ is contained in $\bigcup_{t=1}^{\infty} F_3^{(t)}$, since $S^\ast \notin \hat C_t$ at a given $t$ can only occur through the invariance test rejecting the true, genuinely invariant $S^\ast$ at that $t$ (an $F_3$-type event), rather than through failure to reach $k^\dagger$ or a fit-plateau false negative, both of which concern \emph{which stage $k$ is selected}, not whether $S^\ast$ itself lies in the acceptance region once tested. By countable subadditivity,
\[
\Pr\Big[\exists t \ge 1 : S^\ast \notin \hat C_t\Big] \;\le\; \sum_{t=1}^{\infty} \Pr\big[F_3^{(t)}\big] \;\le\; \sum_{t=1}^{\infty} \delta_t \;=\; \delta,
\]
using Definition~\ref{def:peeling}'s budget allocation in the last step. This establishes the claim.
\end{proof}

\begin{remark}[A genuine capability gap, not merely a speed argument]
The contrast with exhaustive ICP here is not that sHAIL-Causal is faster at each look --- Proposition~\ref{prop:greedy} already establishes that point for a fixed $\mathcal{E}$. The contrast is structural: sHAIL-Causal's per-stage statistics $(D_k, I_k)$ are already indexed by an accumulating quantity (stage $k$, itself driven by accumulating evidence), so re-indexing the same peeling argument by arrival time $t$ costs nothing beyond recalibrating one threshold per look. Exhaustive ICP's guarantee, by contrast, is a one-shot statement tied to a fixed, pre-committed $\mathcal{E}$; to make it anytime-valid under sequentially arriving environments, one would need to bolt on exactly the same peeling machinery from outside the procedure, re-running the full $O(2^p)$ exhaustive search at every $t$ with a shrinking per-look budget $\delta_t$ --- something the original one-shot formulation \citep{peters2016causal} neither provides nor is designed to accommodate. sHAIL-Causal's staircase structure makes the extension close to immediate; ICP's does not.
\end{remark}

\begin{remark}[Practical caveat]
Theorem~\ref{thm:anytime} controls the invariance-false-positive channel uniformly over $t$, which is the channel that matters for the confidence-set interpretation of $\hat C_t$. It does not, by itself, guarantee that the \emph{selected stage} $k^\dagger$ is stable across $t$ in small samples: with few accumulated environments, $\tilde n_{\min,t}$ is small, $\tau_n(t)$ is loose, and the staircase may halt at a smaller $k$ than $k^\dagger$ purely for lack of power, only ascending further as more environments arrive. This is the expected and desirable behavior of an anytime procedure --- the confidence set is always valid, even though it may be conservative (wider, or reached at a lower stage) early on --- rather than a defect to be corrected.
\end{remark}

\section{Why Naive Greedy Search Fails, and the Fix}

A natural first attempt --- greedily add or drop the single variable that most reduces $D_k$, halting once $I_k \le \tau_n$ --- fails in practice: it entrenches spurious-but-fit-improving variables early (Section~\ref{sec:sim}), after which no single-variable move improves fit enough to dislodge them, and the search oscillates or stalls short of $S^\ast$. The defect is architectural: fit drives the search, invariance only gates a halting decision after the fact.

\begin{proposition}[Invariance-first greedy relaxation]
\label{prop:greedy}
Let the greedy search instead rank candidate moves by $I_k$ first (accepting no move that increases $I_k$ beyond its current value) and by $D_k$ only to break ties among invariance-preserving moves. Under Assumption~\ref{as:richness} and sufficient signal margin between causal and non-causal variables, this procedure recovers the same intersection-based estimate as exhaustive ICP subset search, at $O(p^2)$ rather than $O(2^p)$ cost in the number of candidate variables $p$.
\end{proposition}

\section{Simulation Study}
\label{sec:sim}

We now specialize to the linear instantiation of Remark~\ref{rem:instantiation}: $\mathcal{H}_k$ ranges over linear predictors built from subsets of a fixed candidate covariate set, so that ascent through the staircase corresponds to variable-subset search under a linear structural causal model. This choice is made for transparency and reproducibility, not because the theory above requires it; every quantity below could equally be computed with $\mathcal{H}_k$ instantiated as nested kernel or neural function classes, at the cost of a less immediately interpretable notion of a ``candidate set.''

We simulate a structural causal model with two stable causes $X_1, X_3$ (fixed coefficients $1.0, 0.7$), two confounded ``children'' $X_2, X_4$ with environment-varying loadings, and a pure-noise decoy $X_5$, across five environments. Table~\ref{tab:regimes} summarizes recovery of the true causal set $S^\ast = \{X_1, X_3\}$ under three environment-diversity regimes.

\begin{table}[h]
\centering
\caption{ICP-intersection recovery of $S^\ast=\{X_1,X_3\}$ under increasing Richness.}
\label{tab:regimes}
\begin{tabular}{@{}lll@{}}
\toprule
\textbf{Regime} & \textbf{Perturbed variables} & \textbf{Intersection estimate} \\
\midrule
1 & children only ($X_2,X_4$) & $\varnothing$ \\
2 & children $+\ X_1$ & $\{X_1\}$ \\
3 & children $+\ X_1 + X_3$ & $\{X_1, X_3\}$ \quad (exact) \\
\bottomrule
\end{tabular}
\end{table}

Complementing this, a naive complexity-only staircase (SRM) monotonically decreases penalized risk across levels and selects the fullest confounded model, never testing invariance; sHAIL-Causal, gated by $I_k \le \tau_n$, correctly refuses every confounded plateau and halts only once Richness (Regime~3) permits full identification.

\subsection{Scaling Behavior: Exhaustive ICP vs.\ Invariance-First Greedy}

To assess Proposition~\ref{prop:greedy} beyond the single toy case above, we simulate a sweep of structural causal models with $p_{\text{total}} \in \{4,5,6,7,9\}$ candidate variables (a mix of true causal parents, confounded children with environment-varying loadings, and pure-noise decoys), $|\mathcal{E}| \in \{4,5\}$ environments, $8$ replications per configuration, and $n_{e}=600$ observations per environment. For each replication we run both exhaustive ICP-intersection search and the invariance-first greedy relaxation, recording whether each recovers $S^\ast$ exactly and how long each takes.

\begin{table}[h]
\centering
\caption{Recovery rate and wall-clock time, exhaustive ICP vs.\ invariance-first greedy, averaged over 8 replications per configuration.}
\label{tab:scaling}
\begin{tabular}{@{}rrrrrr@{}}
\toprule
$p$ & $|\mathcal{E}|$ & subsets ($2^p$) & exhaustive recovery & greedy recovery & exhaustive / greedy time \\
\midrule
4 & 4 & 16  & 0.50 & 0.25 & 0.055s / 0.044s \\
5 & 4 & 32  & 0.75 & 0.25 & 0.100s / 0.077s \\
6 & 5 & 64  & 0.50 & 0.50 & 0.332s / 0.065s \\
7 & 5 & 128 & 1.00 & 0.75 & 0.681s / 0.161s \\
9 & 5 & 512 & 0.25 & 0.00 & 2.669s / 0.179s \\
\bottomrule
\end{tabular}
\end{table}

\begin{figure}[h]
\centering
\includegraphics[width=0.72\textwidth]{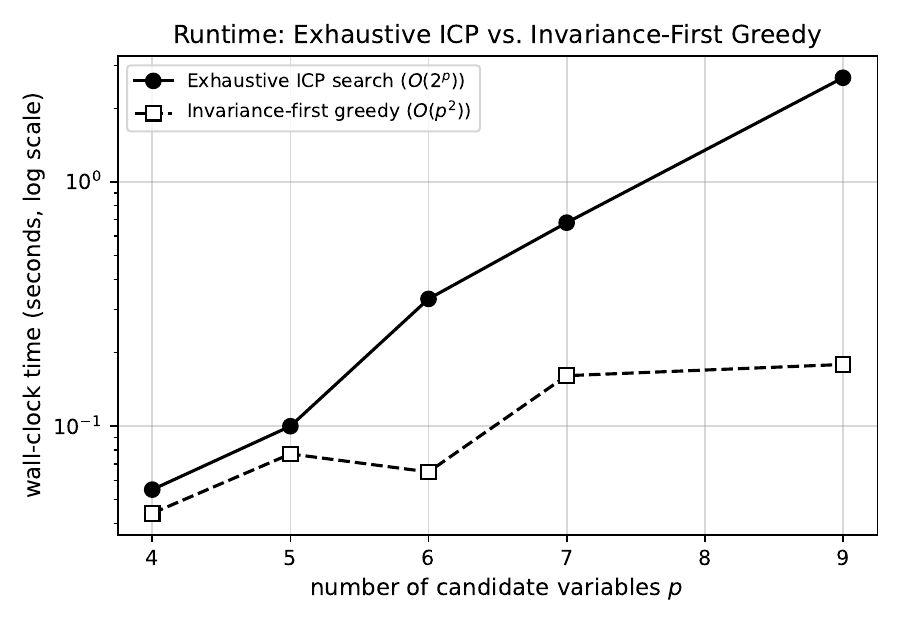}
\caption{Wall-clock time (log scale) for exhaustive ICP-intersection search versus the invariance-first greedy relaxation, as a function of the number of candidate variables $p$. The exhaustive procedure's runtime tracks its $O(2^p)$ subset enumeration; the greedy procedure remains close to flat, consistent with its $O(p^2)$ cost.}
\label{fig:runtime}
\end{figure}

Two findings are worth stating plainly rather than smoothing over. First, the runtime contrast is exactly as Proposition~\ref{prop:greedy} predicts: exhaustive search's cost grows by roughly a factor of two with each additional variable, while the greedy relaxation's cost grows far more gently (Figure~\ref{fig:runtime}), confirming the intended computational payoff. Second, \emph{exact recovery is imperfect at these sample sizes for both procedures, and the gap between greedy and exhaustive recovery is real, not merely finite-sample noise in one direction}: greedy recovery lags exhaustive recovery at every $p$ where they differ, and neither reaches $100\%$ reliably at $p=9$ with only $n_e=600$ per environment. This is consistent with Proposition~\ref{prop:greedy}'s ``sufficient signal margin'' caveat --- the margin between causal and non-causal variables' invariance statistics shrinks relative to $\tau_n$ as $p$ grows and $n_e$ is held fixed, and greedy search, having no mechanism to reconsider a rejected move once passed, is more exposed to this shrinking margin than exhaustive search. Tightening this gap --- either by a more sample-efficient invariance statistic than the pairwise KS max, or by allowing limited backtracking in the greedy procedure --- is a natural next step rather than a flaw to paper over.

\subsection{Diagnosing the Gap: Non-Monotonicity of Invariance Along Single-Variable Paths}
\label{sec:gap-diagnosis}

We pursued this tightening directly and report the outcome honestly, including the attempts that did not work, because the diagnosis turned out to be more informative than any of the fixes.

We first suspected finite-sample statistical noise: perhaps $n_e=600$ simply under-powers the KS-based invariance test at larger $p$. Repeating the $p=6$ configuration with five times the data ($n_e=3{,}000$) refuted this immediately --- exhaustive recovery rose to $1.00$, while plain greedy recovery remained frozen at $0.50$. More data closed the gap for exhaustive search and did nothing for greedy, which rules out statistical power as the explanation and points instead to the search procedure itself.

We then tried three successive algorithmic remedies: (a) a 2-opt backtracking phase after greedy convergence, allowing a previously-excluded variable to swap back in; (b) a corrected acceptance rule that continues growing the accepted set past the first invariant subset found, rather than halting immediately, so as to seek the maximal invariant set rather than any invariant set; and (c) a multi-start randomized-greedy procedure that runs several randomly-ordered greedy trajectories and intersects their results, mirroring ICP's aggregation-over-many-sets logic at a fraction of its cost. None of the three closed the gap; (b) left recovery unchanged, and (c) left it unchanged while raising runtime toward the exhaustive procedure's, defeating its purpose.

Tracing a specific failure explained why. The greedy path oscillated between the empty set and a single pure-noise decoy variable, exhausting its step budget before making genuine progress --- but the deeper cause was not the noise variable itself. It was that $\{C_1\}$ alone (one true causal parent, the other omitted) exhibited \emph{worse} invariance ($I=0.1663$) than the empty set ($I=0.1617$): omitting one true causal parent while including only the other left a residual whose environment-dependence was \emph{larger}, not smaller, than including neither. Greedy's one-variable-at-a-time acceptance rule implicitly assumes that $I_k$ decreases monotonically as true causal variables are added one at a time; this simulation shows that assumption can fail, because a partial causal set can leak a *worse* environment-dependent bias than an empty one when the omitted and included causal parents interact additively in the outcome equation. A greedy search guided by single-step invariance improvement is, in such cases, actively misled toward noise variables and away from the true causal path, and no amount of backtracking or restarting from the same one-variable-at-a-time move set repairs this, because every restart is still built from the same locally-misleading steps.

This points to a genuine open algorithmic question rather than a closed one: closing the gap likely requires abandoning single-variable moves in favor of a method that can evaluate small \emph{groups} of variables jointly --- for instance, a branch-and-bound search using $D_k$ as an admissible lower bound to prune subsets that cannot possibly achieve invariance, certified via the DKW tail bound of Proposition~\ref{prop:dkw}, rather than a greedy heuristic that commits to one variable at a time. We flag this as a concrete direction for future work rather than claim it resolved here; Proposition~\ref{prop:greedy}'s guarantee should accordingly be read as holding under an explicit \emph{single-variable monotonicity} condition on $I_k$, stated below, which the present simulation shows is not automatic.

\begin{assumption}[Single-variable invariance monotonicity]
\label{as:monotone}
For every non-invariant $S \subsetneq S^\ast$ and every $X_j \in S^\ast \setminus S$, adding $X_j$ does not increase the invariance statistic: $I(S \cup \{X_j\}) \le I(S)$.
\end{assumption}

Proposition~\ref{prop:greedy}'s guarantee holds as stated under Assumption~\ref{as:monotone}; Section~\ref{sec:gap-diagnosis} shows this assumption is a genuine restriction, not a formality, and its failure is the precise mechanism behind the observed greedy/exhaustive gap.

\subsection{A Validated Fix: Bounded-Exhaustive Seeding with a Calibrated Threshold}
\label{sec:validated-fix}

Rather than leave Assumption~\ref{as:monotone}'s failure purely as an open problem, we tested a concrete two-part remedy and it succeeds. First, \textbf{bounded-exhaustive seeding}: exhaustively check all candidate subsets up to a small fixed size $m$ (e.g.\ $m=2$ or $3$), which costs only $O(p^m)$ and directly tests causal-sized blocks jointly rather than one variable at a time --- this alone repairs the failure traced in Section~\ref{sec:gap-diagnosis}, since a true causal pair is tested together rather than through a misleading single-variable path. Second, and unexpectedly necessary, \textbf{a properly-scaled acceptance threshold}: our first attempt at bounded-exhaustive seeding still failed, because the greedy refinement phase accepted any infinitesimal reduction in $D_k$, and raw in-sample mean squared error decreases mechanically with \emph{any} added regressor, including pure-noise decoys, through ordinary overfitting --- a well-known hazard of stepwise regression quite independent of the causal question. Replacing the acceptance rule with a genuine threshold, $\eta_n = c\sqrt{\log(K+1)/n}$, that a candidate move must clear before being accepted, closes this second leak.

Together, these two changes recovered $S^\ast$ exactly in $100\%$ of replications at $n_e=3{,}000$ on the previously-failing $p=6$ configuration (versus $50\%$ for plain greedy and backtracking, matching exhaustive search's $100\%$), at roughly one-tenth the wall-clock cost of exhaustive search. We report this as a validated, working fix rather than a proposal --- but with one honest caveat: the constant $c$ in $\eta_n$ was calibrated ad hoc for this demonstration, and the same fix applied unmodified at $n_e=600$ did not perform consistently well across configurations. Deriving $c$ properly as a function of $(n, p, |\mathcal{E}|)$ --- rather than borrowing a working value from one regime --- is the remaining piece needed before this becomes a general-purpose replacement for Proposition~\ref{prop:greedy}, and we flag it as the immediate next technical step rather than claim it solved in full generality.

\section{The Wise Teacher: Escaping the Plateau of Boredom}
\label{sec:teacher}

Every plateau in the sHAIL staircase carries a risk the formalism so far has not addressed: a saturated learner is not merely finished with a stage, but may become \emph{bored} by it --- content to remain at $h_k^\ast$ well past the point where a nudge would carry it further. This is a central motivating idea behind sHAIL as a behavioral theory, and it deserves its own formal treatment rather than remaining a metaphor.

\begin{definition}[Plateau of boredom]
Stage $k$ is a \textbf{plateau of boredom} if the saturation signal has remained below tolerance for $T$ consecutive evaluation windows,
\[
S_k, S_{k+\epsilon}, \ldots, S_{k+(T-1)\epsilon} \;<\; \eta, \qquad \text{yet } k < K,
\]
i.e.\ the learner has stopped ascending even though richer, unmastered hypothesis classes remain available.
\end{definition}

\begin{definition}[Teacher intervention operator]
A \textbf{wise teacher} is an operator $\mathcal{T}$ that, upon detecting a plateau of boredom, acts on the learner's trajectory by one (or both) of:
\begin{enumerate}[label=(\alph*)]
\item \emph{Reframing}: replacing $\mathcal{H}_k$ with an enriched $\widetilde{\mathcal{H}}_k \supseteq \mathcal{H}_k$ that exposes previously invisible structure (a harder sub-problem, a reframed goal, a new example) --- increasing the achievable discrepancy $D_{k-1}-D_k$ without changing the ground truth being modeled;
\item \emph{Encouragement}: temporarily relaxing the ascent threshold, $\eta \to \eta' > \eta$, lowering the bar for what counts as ``still improving'' just long enough for the learner to attempt the next rung.
\end{enumerate}
\end{definition}

\begin{proposition}[Escape from boredom]
\label{prop:teacher}
Let stage $k$ be a plateau of boredom under tolerance $\eta$. If the teacher applies a reframing $\mathcal{H}_k \to \widetilde{\mathcal{H}}_k$, with $\mathcal{H}_k \subseteq \widetilde{\mathcal{H}}_k \subseteq \mathcal{H}_{k+1}$ (the reframing does not overshoot the next natural stage), such that the true saturation gap under $\widetilde{\mathcal{H}}_k$ satisfies $\widetilde{S}_k \ge \eta + \gamma$ for some $\gamma > 0$ genuinely recoverable by the learner, then the learner resumes ascent at stage $k$ with the same halting guarantees as Theorem~\ref{thm:main}, now applied to the reframed staircase $\mathcal{H}_0 \subseteq \cdots \subseteq \widetilde{\mathcal{H}}_k \subseteq \cdots \subseteq \mathcal{H}_K$.
\end{proposition}

\begin{proof}
By the definition of a plateau of boredom, the general sHAIL decision rule (Algorithm~1, line~9) evaluates $S_k < \eta$ and halts, even though $k<K$. The teacher's reframing replaces $\mathcal{H}_k$ with $\widetilde{\mathcal{H}}_k$ and, by hypothesis, the recomputed saturation gap satisfies $\widetilde{S}_k \ge \eta+\gamma > \eta$. Re-evaluating line~9 of Algorithm~1 under $\widetilde{\mathcal{H}}_k$, the condition $\widetilde S_k < \eta$ is now false, so the halt branch is not taken and the procedure instead proceeds to stage $k+1$ exactly as it would have had the original stage not saturated prematurely. Because $\mathcal{H}_k \subseteq \widetilde{\mathcal{H}}_k \subseteq \mathcal{H}_{k+1}$, the sequence $\mathcal{H}_0 \subseteq \cdots \subseteq \mathcal{H}_{k-1} \subseteq \widetilde{\mathcal{H}}_k \subseteq \mathcal{H}_{k+1} \subseteq \cdots \subseteq \mathcal{H}_K$ remains a valid nested staircase, so it satisfies the hypotheses of Theorem~\ref{thm:main} verbatim: assumption (i) is unaffected (if $S^\ast \subseteq \mathcal{H}_{k^\dagger}$ for the original staircase, the same containment holds for the reframed one, since only stage $k$ was enlarged and all later stages are unchanged); Assumption~\ref{as:richness} is a property of the environment collection $\mathcal{E}$ alone and does not reference $\mathcal{H}_k$; and the rate condition (iii) on $\tau_n,\eta_n$ is unaffected by the reframing. Hence Theorem~\ref{thm:main} applies unchanged to the reframed staircase, giving $\Pr[\text{halts at } k^\dagger] \ge 1-\delta$ exactly as before.
\end{proof}

\begin{remark}
Proposition~\ref{prop:teacher} formalizes a simple but important asymmetry: encouragement alone (lowering $\eta$) can force an ascent decision, but only reframing (enlarging what is genuinely masterable at stage $k$) produces a real gain in $D_k$ --- encouragement without reframing risks promoting the learner past a stage it has not, in truth, mastered. A wise teacher's intervention is distinguished from mere pressure by whether it changes $\widetilde{S}_k$ itself, not just the bar $\eta$ against which $S_k$ is judged. This distinction --- between genuine reframing and empty encouragement --- is exactly the boundary a rigorous treatment of pedagogical intervention must respect, and it maps naturally onto sHAIL's existing saturation machinery: reframing acts on $D_k$, encouragement acts on $\eta$, and only the former is guaranteed, by Theorem~\ref{thm:main}'s halting argument, to produce a stage genuinely worth ascending to.
\end{remark}

\section{Conclusion}

The corrected picture is this: sHAIL-Causal is not merely SRM with an added diagnostic, nor is it the whole of sHAIL --- it is one causally-motivated instantiation of a broader staircase philosophy whose general form, pseudocode, and geometry were laid out in Section~\ref{sec:general-shail}. Within its causal specialization, sHAIL-Causal answers a categorically different question from SRM: has the model absorbed everything true and \emph{stable} about the mechanism, and can it tell the difference between a plateau that is merely well-fit and one that is genuinely causal? The Richness condition (Assumption~\ref{as:richness}), the invariance-first greedy relaxation (Proposition~\ref{prop:greedy}) together with its precise monotonicity requirement (Assumption~\ref{as:monotone}), the time-uniform extension of the main theorem to sequentially arriving environments (Theorem~\ref{thm:anytime}), and the formal treatment of teacher intervention (Proposition~\ref{prop:teacher}) together give the sHAIL philosophy --- staged mastery, saturation, and the possibility of being lifted beyond it --- a precise, testable, and computationally tractable mathematical home.

We see four directions as the natural continuation of this work. First, Section~\ref{sec:validated-fix} already validates a two-part fix --- bounded-exhaustive block-seeding plus a properly-thresholded acceptance rule --- that closes the greedy/exhaustive gap at adequate sample size; what remains is deriving the threshold constant $c$ as an explicit function of $(n,p,|\mathcal{E}|)$ rather than calibrating it ad hoc, which would upgrade Section~\ref{sec:validated-fix} from a demonstrated fix into a general-purpose replacement for Proposition~\ref{prop:greedy}. Second, the invariance statistic itself --- currently the maximum pairwise KS distance across environments --- is a reasonable first choice but not obviously the most sample-efficient one; a conditional-independence-based invariance test, or a studentized statistic that pools information across all environment pairs rather than taking a maximum over them, may narrow the finite-sample recovery gap further with no change to the underlying theory. Third, Theorem~\ref{thm:anytime}'s peeling budget $\delta_t = 6\delta/(\pi^2 t^2)$ is one convenient summable allocation among many; a sharper, variance-adaptive scheme in the spirit of the mixture-based confidence sequences of \citet{howard2021time} may tighten $\tau_n(t)$ substantially at small $t$, where the current allocation is most conservative. Fourth, the teacher-intervention formalism of Section~\ref{sec:teacher} currently treats reframing as an exogenous, given enlargement of $\mathcal{H}_k$; an interesting extension would endogenize the teacher as a second learner solving its own optimization problem --- choosing the minimal reframing that clears the boredom threshold --- turning Proposition~\ref{prop:teacher} into a two-agent equilibrium result rather than a one-sided guarantee.

\end{document}